\documentclass[11pt]{article}
\pdfoutput=1
\usepackage{enumerate}
\usepackage{pdfsync}
\usepackage[OT1]{fontenc}
\PassOptionsToPackage{numbers,sort&compress}{natbib}
\usepackage{smile}
\usepackage[colorlinks,
            linkcolor=red,
            anchorcolor=blue,
            citecolor=blue
            ]{hyperref}
\usepackage{color, colortbl}
\usepackage{fullpage}
\usepackage{hyperref}
\usepackage[protrusion=true,expansion=true]{microtype}
\usepackage{pbox}
\usepackage{setspace}
\usepackage{tabularx}
\usepackage{float}
\usepackage{wrapfig,lipsum}
\usepackage{enumitem}
\usepackage{microtype}
\usepackage{graphicx}
\usepackage{subfigure}
\usepackage{enumerate}
\usepackage{enumitem}
\usepackage{pgfplots}
\usetikzlibrary{arrows,shapes,snakes,automata,backgrounds,petri}
\usepackage{booktabs} 

\newtheorem{condition}[theorem]{Condition}

\usepackage{enumitem}

\makeatletter
\newcommand*{\rom}[1]{\expandafter\@slowromancap\romannumeral #1@}
\makeatother
\title{\huge Sharp First-Order Lower Bounds for Higher-Order Smooth Nonconvex Optimization}

\author
{
	Dongruo Zhou\thanks{Department of Computer Science, Indiana University Bloomington, IN 47408, USA; e-mail: {\tt dz13@iu.edu}}  
}

\begin{document}
\date{}
\maketitle

\begin{abstract}
We study the deterministic first-order oracle complexity of finding \(\epsilon\)-stationary points in smooth nonconvex optimization when the objective satisfies higher-order smoothness assumptions. While the classical \(\epsilon^{-2}\) rate is optimal under only Lipschitz gradients, higher-order smoothness leads to accelerated first-order upper bounds, most notably the \(\epsilon^{-7/4}\) rate under Lipschitz Hessians and the \(\epsilon^{-5/3}\) rate under Lipschitz third derivatives. The matching lower bounds, however, have remained open. We resolve this gap by proving a new dimension-free first-order lower bound for higher-order smooth nonconvex functions, valid for every finite smoothness order. In particular, our construction gives a matching \(\Omega(\epsilon^{-7/4})\) lower bound in the Hessian-Lipschitz case and a matching \(\Omega(\epsilon^{-5/3})\) lower bound in the third-order-smooth regime. The hard instance is based on a \emph{block-chain} mechanism that enforces blockwise oracle revelation while preserving the smoothness structure needed for the scalar hard instance. The lower-bound construction was discovered with the assistance of ChatGPT 5.5 Pro and subsequently verified by the authors.
\end{abstract}

\section{Introduction}

Finding approximate stationary points is one of the basic tasks in smooth nonconvex optimization. Given a differentiable objective \(f\), the goal is to find a point \(x\) with 
\begin{align}
\|\nabla f(x)\|\le\epsilon.\notag
\end{align}
We focus on first-order methods, which only query function values and gradients, because they are the most widely used methods in large-scale nonconvex optimization. For functions with \(L_1\)-Lipschitz gradients and initial gap \(\Delta\), gradient descent finds such a point in \(O(\Delta L_1\epsilon^{-2})\) first-order oracle calls. This rate is worst-case optimal under only gradient Lipschitzness, so in the general smooth function class there is no room for acceleration beyond the classical \(\epsilon^{-2}\) dependence.

\begin{table}[t]
\centering
\caption{Known dimension-free deterministic first-order rates under higher-order smoothness. Only the dependence on \(\epsilon\) is shown. The bound of \citet{xiong2026parameterfree} counts gradient evaluations; its backtracking also uses function values.}
\label{tab:first-order-epsilon-dependence}
\small
\begin{tabular}{@{}lccc@{}}
\toprule
 & \(p=1\) & \(p=2\) & \(p\ge3\) \\
\midrule
Upper bound & \makecell[c]{\(\epsilon^{-2}\)\\(gradient descent)} & \makecell[c]{\(\epsilon^{-7/4}\)\\\citep{agarwal2017finding,carmon2018acceleratedmethods,jin2018acceleratedsaddle,zhang2021escapesaddle,carmon2017convexuntilproven,li2023restartedagd,marumo2024parameterfreeagd}} & \makecell[c]{\(\epsilon^{-5/3}\)\\\citep{carmon2017convexuntilproven,xiong2026parameterfree}} \\
\cmidrule(lr){1-4}
Lower bound & \makecell[c]{\(\epsilon^{-2}\)\\\citep{carmon2020lowerfirstorderI}} & \makecell[c]{\(\epsilon^{-12/7}\)\\\citep{carmon2017lowerfirstorder}} & \makecell[c]{\(\epsilon^{-8/5}\)\\\citep{carmon2017lowerfirstorder}} \\
\cmidrule(lr){1-4}
\textbf{Our result} & \rule{1.5em}{0.4pt} & \makecell[c]{\(\boldsymbol{\epsilon^{-7/4}}\)\\(Theorem~\ref{thm:first-order-lower-bound})} & \makecell[c]{\(\boldsymbol{\epsilon^{-5/3}}\)\\(Theorem~\ref{thm:first-order-lower-bound})} \\
\bottomrule
\end{tabular}
\end{table}

A natural way to seek faster rates is to impose higher-order smoothness on the objective. The first and most important case is second-order smoothness, namely Lipschitz continuity of the Hessian: it is the weakest standard assumption beyond gradient Lipschitzness that controls how curvature changes, and it gives algorithms enough structure to distinguish convex-like regions from genuinely nonconvex ones. Algorithmic evidence for acceleration under this assumption comes from several related oracle models and target guarantees. One line of work uses Hessian-vector products to obtain curvature-aware guarantees, often with the stronger goal of finding approximate local minima. For example, \citet{agarwal2017finding} approximately solve cubic-regularized subproblems and find approximate local minima within \(O(\epsilon^{-7/4})\) Hessian-vector-product complexity. In a related Hessian-vector-product-based accelerated framework, \citet{carmon2018acceleratedmethods} obtain \(O(\epsilon^{-7/4})\) complexity for finding first-order stationary points. Within the first-order saddle-escaping literature, accelerated and negative-curvature-based gradient methods yield \(O(\epsilon^{-7/4})\) complexity bounds in the Hessian-Lipschitz setting \citep{jin2018acceleratedsaddle,zhang2021escapesaddle}.

The line closest to our setting keeps both the oracle and the target purely
first-order. The accelerated ``convex until proven guilty'' framework of
\citet{carmon2017convexuntilproven} runs accelerated gradient in regions that
behave convexly and uses failure of convexity as a certificate for progress in
the nonconvex landscape, giving \(O(\epsilon^{-7/4}\log(1/\epsilon))\)
first-order complexity under Lipschitz Hessians and
\(O(\epsilon^{-5/3}\log(1/\epsilon))\) under Lipschitz third derivatives.
Subsequent work refined this picture: restarted accelerated gradient removes
the logarithmic factor in the Hessian-Lipschitz case
\citep{li2023restartedagd}, while parameter-free accelerated methods recover
the \(O(\epsilon^{-7/4})\) complexity without requiring the same
smoothness-parameter tuning \citep{marumo2024parameterfreeagd}. In the
third-order regime, \citet{xiong2026parameterfree} develop PF-AGD, which
attains an \(O(\epsilon^{-5/3}\log(1/\epsilon))\) gradient-evaluation bound
without prior knowledge of \(L_1\) or \(L_3\); the method still uses function
values for backtracking. \citet{jiang2024quasinewton} give a gradient-only
quasi-Newton method with \(O(d^{1/4}\epsilon^{-13/8})\) complexity, improving
the \(\epsilon\)-dependence in low dimension but with explicit dimension
dependence rather than a dimension-free guarantee. Thus higher-order
smoothness does allow acceleration beyond \(\epsilon^{-2}\), even when the
algorithm only observes function values and gradients. The absence of matching
dimension-free lower bounds has been explicitly noted in later work that
sharpened or reinterpreted the upper-bound landscape: even after the
logarithmic factor was removed, and even after online-learning viewpoints
recovered parts of the upper-bound picture, the dimension-free deterministic
first-order lower bound remained unchanged
\citep{li2023restartedagd,cutkosky2023online}.

Despite these upper-bound developments, whether the accelerated rates were optimal remained unclear. On the lower-bound side, the best previous result is due to \citet{carmon2017lowerfirstorder}, who proved an \(\Omega(\epsilon^{-12/7})\) lower bound in the Hessian-Lipschitz setting and an \(\Omega(\epsilon^{-8/5})\) lower bound for arbitrarily smooth functions. This left an exponent gap of \(1/28\) in the Hessian-Lipschitz case, between \(12/7\) and \(7/4\), and an exponent gap of \(1/15\) relative to the \(5/3\) first-order rate under Lipschitz third derivatives. Whether these gaps reflected a limitation of the known lower-bound constructions or a genuine separation has remained a long-standing open problem. This motivates the following question.

\begin{center}
\emph{What is the optimal first-order oracle complexity of finding \(\epsilon\)-stationary points when higher-order smoothness is available?}
\end{center}

In this work, we answer this question by proving a new lower bound, which resolves a long-standing open gap in deterministic first-order oracle complexity. Our contributions are as follows.

\begin{itemize}
\item As summarized in Table~\ref{tab:first-order-epsilon-dependence}, we give a new dimension-free first-order lower-bound construction for higher-order smooth nonconvex functions, valid for every finite smoothness order \(p\). In the Hessian-Lipschitz case, the construction gives the matching \(\Omega(\epsilon^{-7/4})\) dependence, closing the gap between the previous \(\epsilon^{-12/7}\) lower bound and the best known first-order upper bound. Under Lipschitz third derivatives, it gives the matching \(\Omega(\epsilon^{-5/3})\) dependence in the third-order-smooth regime. The full theorem tracks the dependence on \(\Delta\), \(L_1,\ldots,L_p\), and the active smoothness constraints.
\item Technically, we introduce a block-chain hard instance that augments the scalar transition problem with a \emph{linear operator} for observation. The nonconvex potential is applied after a linear readout, whose operator norm provides an additional knob for calibrating higher-order smoothness while preserving the scalar residual certificate. To keep this pullback compatible with first-order revelation, each scalar coordinate is replaced by a chain block: the previous scalar state enters through an entry coordinate, whereas the current scalar state is read through a suffix projection. This separation lets a zero-respecting method reveal the construction only layer by layer, yielding a block-layer tail-hiding property and a gradient lower bound before the hidden tail is reached.
\end{itemize}

\paragraph{Notation.}
For $n\in\mathbb N$, write $[n]=\{1,\ldots,n\}$. We use $\langle\cdot,\cdot\rangle$ for the Euclidean inner product, $\|\cdot\|$ for the Euclidean norm of a vector and the induced operator norm of a matrix, and $\operatorname{supp}(x)$ for the support of a vector $x$. The vectors \(e_1,e_2,\ldots\) denote the standard basis vectors in the ambient Euclidean space; when the dimension matters, we write \(e_j^{(m)}\in\mathbb R^m\). For a \(C^q\) function \(f\), \(D^q f(x)\) denotes its \(q\)-th Fr\'echet derivative at \(x\), viewed as a symmetric \(q\)-linear form. For a \(q\)-linear form \(A\), we write
\[
    \|A\|_{\rm op}
    :=
    \sup_{\|v_1\|,\ldots,\|v_q\|\le1}|A[v_1,\ldots,v_q]|.
\]
For a map \(G\) between Euclidean spaces, \(\operatorname{Lip}(G)\) denotes its global Lipschitz constant with respect to the induced Euclidean norms; in particular, \(\operatorname{Lip}(D^q f)\) uses the above operator norm on \(q\)-linear forms. We use $O(\cdot)$, $\Omega(\cdot)$, and $\Theta(\cdot)$ for asymptotic comparisons; all constants in the main finite-tolerance bound are displayed explicitly.
Named numerical constants introduced below, such as the constants \(\ell_q\) in
the scalar primitive bounds, are displayed explicitly in quantitative rates
rather than absorbed into this asymptotic notation.

\section{Related Work}

The literature surrounding first-order nonconvex optimization and higher-order
smoothness is broad, so we discuss only the strands most directly related to our
lower-bound result; the references below are necessarily selective.

\paragraph{Higher-order oracle models.}
A separate line of work studies algorithms that can query higher-order derivative information, rather than only function values and gradients. Classical trust-region methods and regularized Newton methods form the basic second-order toolkit \citep{conn2000trustregion,cartis2010complexity}. In the Hessian-Lipschitz setting, cubic regularization builds a local quadratic model with a cubic penalty and achieves the \(O(\epsilon^{-3/2})\) complexity for finding first-order stationary points \citep{nesterov2006cubic,cartis2010complexity}; related work also characterizes the complexity of reaching second-order stationarity \citep{cartis2012secondorder}. More generally, regularized \(p\)th-order Taylor methods achieve \(O(\epsilon^{-(p+1)/p})\) under Lipschitz \(p\)th derivatives \citep{birgin2017worstcase}, and worst-case analyses for second-order methods further clarify the optimality of this oracle model \citep{cartis2017worstcase}. These rates are matched by lower bounds for algorithms with access to derivatives up to order \(p\) \citep{carmon2020lowerfirstorderI}. Recent mixed-oracle work further studies tradeoffs between gradient and Hessian queries; in particular, finite-difference Hessian approximations yield improved low-dimensional gradient complexity, and even a small number of Hessian queries can change the gradient-query exponent \citep{adil2025balancing}. This higher-order-oracle theory is complementary to our setting: we impose higher-order smoothness on the objective, but the algorithm remains first-order.

\paragraph{Stochastic, finite-sum, and local-minimum variants.}
There is also a substantial stochastic, finite-sum, and approximate-local-minimum counterpart to the oracle-complexity literature. For second-order stationarity, early stochastic saddle-escaping results show that noisy or perturbed gradient methods can avoid strict saddles and reach approximate local minima \citep{ge2015escaping,jin2017escapesaddle,jin2019nonconvexml}. Under Hessian-Lipschitz structure, Hessian-vector and gradient-only algorithms obtain faster approximate-local-minimum guarantees through cubic regularization, accelerated saddle escaping, and negative-curvature extraction \citep{agarwal2017finding,carmon2018acceleratedmethods,jin2018acceleratedsaddle,zhang2021escapesaddle,allenzhu2018neon2,liu2018adaptive,xu2017neonplus}. Stochastic and finite-sum refinements combine perturbed SGD, variance reduction, stochastic cubic regularization, and third-order smoothness to improve sample or gradient complexity for local-minimum finding \citep{xu2018firstorderstochastic,fang2018spider,fang2019sharp,li2019ssrgd,tripuraneni2018stochasticcubic,yu2018thirdorder,zhou2018findinglocalminima,zhou2018svrcubic,zhou2020srvrn,zhou2020snvr}. More recently, an online-to-nonconvex conversion viewpoint gives deterministic and stochastic upper bounds through a common framework \citep{cutkosky2023online}. On the lower-bound side, \citet{arjevani2020secondorderstochastic} characterize the power and limitations of stochastic Hessian-vector and higher-order oracles, showing that second-order information does not remove the intrinsic stochastic barrier in the same way it can in deterministic optimization. In finite-sum models, lower bounds for higher-order smooth nonconvex objectives were developed by \citet{emmenegger2021higherorderfinitesum} and then sharpened in a dimension-free form by \citet{zhou2022dimensionfreefinitesum}. In the bounded-variance stochastic-gradient model, \citet{arjevani2022lowerstochastic} prove tight lower bounds showing that SGD is minimax optimal for finding first-order stationary points, while stronger mean-squared smoothness assumptions recover the variance-reduction rate. These works address sample, noise, component-oracle, or saddle-escaping complexity; they are therefore complementary to the deterministic, dimension-free full-gradient lower bound studied here.

\section{Preliminaries}\label{sec:preliminaries}
We follow the oracle-complexity framework of
\citet{carmon2020lowerfirstorderI,carmon2017lowerfirstorder}.

\paragraph{Function classes.}
Following Def.~1 of \citet{carmon2020lowerfirstorderI}, let \(p\ge1\) and
\(L_p>0\). We say that \(f:\mathbb R^d\to\mathbb R\) has \(L_p\)-Lipschitz
\(p\)th-order derivatives if it is \(p\) times continuously differentiable and,
for every \(x\in\mathbb R^d\) and every \(v\in\mathbb R^d\) with \(\|v\|=1\),
the directional projection \(f_{x,v}(t):=f(x+tv)\) satisfies
\[
\left|
f_{x,v}^{(p)}(t)-f_{x,v}^{(p)}(t')
\right|
\le
L_p|t-t'|
\qquad
\text{for all }t,t'\in\mathbb R.
\]
Let $\Delta>0$. The class $\mathcal F_p(\Delta,L_p)$ denotes the union, over $d\in\mathbb N$, of all $C^\infty$ functions $f:\mathbb R^d\to\mathbb R$ with $L_p$-Lipschitz $p$th-order derivatives and $f(0)-\inf_x f(x)\le\Delta$. For positive $L_1,\ldots,L_p$, define
\[
\mathcal F_{1:p}(\Delta,L_1,\ldots,L_p)
:=
\bigcap_{q=1}^p \mathcal F_q(\Delta,L_q).
\]

\paragraph{First-order oracle complexity.}
As in Eq.~(7) of \citet{carmon2020lowerfirstorderI}, a deterministic
first-order method \(\mathsf A\) starts from \(x^{(0)}=0\) and, when applied to
\(f:\mathbb R^d\to\mathbb R\), generates iterates \(x^{(t)}\in\mathbb R^d\) such
that \(x^{(t+1)}\) is a deterministic function of the previous oracle answers
\[
\left(f(x^{(0)}),\nabla f(x^{(0)}),\ldots,f(x^{(t)}),\nabla f(x^{(t)})\right).
\]
Define
\[
T_\epsilon(\mathsf A,f)
:=
\inf\{t\ge0:\|\nabla f(x^{(t)})\|\le\epsilon\}.
\]
We use the convention $\inf\emptyset=\infty$.
For a class of deterministic first-order methods $\mathcal A$ and a function class $\mathcal F$, write
\[
T_\epsilon(\mathcal A,\mathcal F)
:=
\inf_{\mathsf A\in\mathcal A}\sup_{f\in\mathcal F}T_\epsilon(\mathsf A,f),
\]
up to the immaterial additive constant coming from whether the initial point is counted as a query.

\paragraph{Zero-respecting methods.}
Following Eq.~(5) of \citet{carmon2017lowerfirstorder}, a sequence
\(\{x^{(t)}\}_{t\ge0}\) is first-order zero-respecting with respect to \(f\) if
\[
\operatorname{supp}(x^{(t)})
\subseteq
\bigcup_{s<t}\operatorname{supp}(\nabla f(x^{(s)}))
\qquad\text{for all }t\ge0.
\]
A first-order method is zero-respecting if, for every \(f\), its generated
iterate sequence is zero-respecting with respect to \(f\). By contrast, a
function $f:\mathbb R^d\to\mathbb R$ is a classical first-order zero-chain
\citep[Def.~3]{carmon2017lowerfirstorder} if, for every $i\in[d]$ and
$x\in\mathbb R^d$,
\[
\operatorname{supp}(x)\subseteq\{1,\ldots,i-1\}
\quad\Longrightarrow\quad
\operatorname{supp}(\nabla f(x))\subseteq\{1,\ldots,i\}.
\]

\paragraph{Primitive functions.}
We also fix two elementary primitives used in the construction. For
\(m\in\mathbb N\), let \(\mathsf P_m\in\mathbb R^{m\times m}\) be the path
Laplacian in tridiagonal form, and let \(H_{\gamma,L}^{(m)}\) be the
corresponding massive path operator:
\begin{equation}
\label{eq:path-primitive}
\mathsf P_1:=0,\qquad
\mathsf P_m:=
\begin{pmatrix}
1 & -1 &        &        &   \\
-1& 2  & -1     &        &   \\
  & \ddots & \ddots & \ddots &   \\
  &        & -1 & 2  & -1\\
  &        &    & -1 & 1
\end{pmatrix}
\quad(m\ge2),\qquad
H_{\gamma,L}^{(m)}:=\gamma I_m+L\mathsf P_m .
\end{equation}
For \(r\ge1\), let \(\Upsilon_r\) be the scalar potential from Eq.~(10) of
\citet{carmon2017lowerfirstorder}:
\begin{equation}
\label{eq:scalar-potential}
\Upsilon_r(x):=
120\int_1^x \frac{t^2(t-1)}{1+(t/r)^2}\,dt .
\end{equation}

We collect the elementary bounds on these two primitives for later use.

\begin{proposition}[Primitive bounds]\label{prop:primitive-bounds}
The primitives in \eqref{eq:path-primitive} and \eqref{eq:scalar-potential}
satisfy the following bounds.
\begin{itemize}[leftmargin=*,itemsep=2pt,topsep=2pt]
\item For every \(m\in\mathbb N\) and \(\gamma,L>0\),
\(0\preceq \mathsf P_m\preceq 4I\) and
\(\gamma I\preceq H_{\gamma,L}^{(m)}\preceq(\gamma+4L)I\).
\item We use the explicit constants
\[
    \ell_1:=360,\qquad \ell_q:=240q!\quad(q\ge2).
\]
For every \(r\ge1\), we have
\(\Upsilon_r\in C^\infty(\mathbb R)\),
\(\Upsilon_r'(0)=\Upsilon_r'(1)=0\),
\(m_{\Upsilon_r}:=\inf_x\Upsilon_r(x)=\Upsilon_r(1)=0\), and
\(\Upsilon_r(0)\le10\). Moreover, for every \(q\ge1\),
\[
    \operatorname{Lip}\bigl(\Upsilon_r^{(q)}\bigr)
    \le
    \ell_q r^{3-q},
    \qquad
    \left\|\Upsilon_r^{(q+1)}\right\|_\infty
    \le
    \ell_q r^{3-q}.
\]
\end{itemize}
\end{proposition}

\begin{proof}
The path bound follows from the quadratic form of the path Laplacian.
For the potential, put \(t=x/r\) and write
\[
\Upsilon_r'(x)=120\left(r^3\frac{t^3}{1+t^2}
-r^2\frac{t^2}{1+t^2}\right).
\]
The identities \(t^3/(1+t^2)=t-t/(1+t^2)\) and
\(t^2/(1+t^2)=1-1/(1+t^2)\), followed by partial fractions at
\(\pm i\), show that for real \(t\) every derivative of order \(q\ge1\)
of either \(t/(1+t^2)\) or \(1/(1+t^2)\) has absolute value at most
\(q!\). Hence
\(\|\Upsilon_r^{(q+1)}\|_\infty
\le120r^{3-q}(\mathbf 1_{\{q=1\}}+2q!)\), which gives the displayed
\(\ell_q\). The remaining potential properties follow directly from
\eqref{eq:scalar-potential}; in particular
\(\Upsilon_r(0)\le120\int_0^1t^2(1-t)\,dt=10\).
\end{proof}

\section{Block-Chain Hard Instance}\label{sec:hard-instance}
We now build the hard instance, starting from the limitation of the scalar-chain template used in prior first-order lower bounds.

\paragraph{Existing lower bound revisited.}
We begin by recalling the scalar chain underlying the first-order lower bound of
\citet{carmon2017lowerfirstorder}. For a fixed \(r\ge1\), its original
unscaled form has a multiplier \(0<\mu\le1\) on the scalar nonconvex potential:
\[
    \phi_\mu(u)
    :=
    \frac{\sqrt\mu}{2}(u_1-1)^2
    +
    \frac12\sum_{s=1}^{N}(u_{s+1}-u_s)^2
    +
    \mu\sum_{s=1}^{N}\Upsilon_r(u_s),
    \qquad
    u\in\mathbb R^{N+1}.
\]

For \(s\ge2\), the cross term \(-u_{s-1}u_s\) is the only transmission from
phase \(s-1\) to phase \(s\), while the separable potential
\(\mu\Upsilon_r(u_s)\) supplies the on-site nonconvex forcing. The following
residual certificate is the key point used by the lower-bound argument.

\begin{proposition}[Scalar transition lower bound, Lemma 3, \citealt{carmon2017lowerfirstorder}]\label{prop:scalar-transition-lower-bound}
There exists a numerical constant \(c_0:=1/4\) such that, for every \(N\ge1\),
every \(r\ge1\), every \(0<\mu\le1\), and every \(u\in\mathbb R^{N+1}\) with
\(u_N=u_{N+1}=0\), the scalar chain \(\phi_\mu\) defined above satisfies
\[
\|\nabla\phi_\mu(u)\|_2
\ge
c_0\mu^{3/4}.
\]
\end{proposition}

Thus, as long as a zero-respecting method has not reached the tail coordinate,
we have \(u_N=0\), and the scalar transition lower bound certifies a large
gradient norm. To turn this residual certificate into an
\(\epsilon\)-stationarity obstruction under prescribed smoothness budgets,
\citet{carmon2017lowerfirstorder} consider the scaled instance
\(f(x):=\lambda\sigma^2\phi_\mu(x/\sigma)\), with \(N=T\). The parameters
\(\lambda,\sigma,\mu,r\) are then calibrated so that the scaled chain belongs to
the target smoothness class, while the chain length \(T\) is set by the available
function-value budget; with these choices, all points whose tail coordinates
remain hidden still have gradient norm larger than \(\epsilon\).  

\paragraph{\emph{Block-chain} hard instance.}
One key parameter in the scalar hard instance \(\phi_\mu\) is \(\mu\), which
sets the strength of the nonconvex onsite potential relative to the quadratic
chain. This parameter must be tuned carefully in order to satisfy higher-order
smoothness constraints while retaining a nontrivial gradient certificate. By
contrast, \(\lambda\) and \(\sigma\) are global amplitude and length-scale
parameters: they rescale all derivative orders in a coupled way, and therefore
cannot by themselves provide the selective balance needed for higher-order
calibration. \citet{carmon2017lowerfirstorder} also extend the construction
beyond quadratic activations and obtain analogous gradient lower bounds, but
the resulting scalar balancing still leaves a gap between the best known
first-order upper bound \(\epsilon^{-7/4}\) for \(p=2\) and their calibrated
lower bound \(\epsilon^{-12/7}\).

Our construction closes this gap through two observations.

\begin{itemize}[leftmargin=*]

    \item \textbf{Tune smoothness through a linear pullback.}
    Instead of relying only on the scalar coefficient \(\mu\), we introduce an
    additional balancing mechanism by applying the nonconvex potential after a
    linear observation map. At the level of intuition, this amounts to replacing
    the coordinatewise potential by
\begin{align}
\sum_{s=1}^{N}\Upsilon_r(u_s) \rightarrow \sum_{s=1}^{N}\Upsilon_r(a_s^\top u)
\end{align}
where \(u\mapsto(a_1^\top u,\ldots,a_N^\top u)\) is a \emph{linear operator}.
Smoothness constraints are stable under such pullbacks: the relevant derivative
operator norms are multiplied only by powers of the operator norm of this
linear map. Thus the geometry of the observation map gives a new way to tune
higher-order smoothness without changing the scalar residual certificate.

\item \textbf{Separate revelation from observation.}
    An arbitrary linear observation map would destroy the lower-bound argument:
    the gradient of \(\Upsilon_r(a_s^\top u)\) points in the direction \(a_s\),
    which could reveal information outside the zero-chain order. The remedy is
    to separate the coordinates that drive the oracle revelation from the
    coordinates observed by the nonconvex potential. Informally, the construction
    uses a map of the form
\begin{align}
u\rightarrow (u, a_s^\top u),
\end{align}
where the first component keeps the original chain-like revelation structure,
while the second component feeds the nonconvex potential and supplies the
high-order smoothness scaling. The formal block construction below implements
this separation by making the nonconvex readout compatible with the reveal
order.

\end{itemize}

With this intuition in mind, we now define the block chain.

\begin{definition}[Block chain]\label{def:function-class}
Let $M,N\in\mathbb N$, let $\gamma,L,\lambda,\kappa,\tau>0$, and let $a\in\mathbb R^M$ be a nonzero suffix vector, meaning that for some $j_\star\in[M]$,
\(\operatorname{supp}(a)\subseteq\{j_\star,\ldots,M\}\).
Set \(d:=MN\), and identify \(\mathbb R^d\) with \((\mathbb R^M)^N\).
For a block vector \(z=(z_1,\ldots,z_N)\in(\mathbb R^M)^N\), the scalar chain variable in phase \(s\) is the suffix readout
\(u_s:=a^\top z_s\), with \(u_0:=1\). The entry coordinate of phase \(s\) is
\(x_s:=e_1^\top z_s\).
The previous phase enters the next block through \(x_s\), while the current phase is read only through the suffix readout \(u_s\). For \(r\ge1\), define
\begin{align}
f_{M,N,\gamma,L,a,\lambda,\kappa,\tau,r}(z)
&:=
\sum_{s=1}^N
\left[
    \frac12 z_s^\top H_{\gamma,L}^{(M)}z_s
    -
    \lambda u_{s-1}x_s
    +
    \frac{\kappa}{2}u_s^2
    +
    \tau\Upsilon_r(u_s)
\right].
\label{eq:block-chain-expanded}
\end{align}

For the normalized member of this family, set
\begin{equation}
\label{eq:scalar-reduction-geometry}
    e:=e_1,\qquad
    \rho:=e^\top\!\left(H_{\gamma,L}^{(M)}\right)^{-1}e,\qquad
    \alpha:=a^\top\!\left(H_{\gamma,L}^{(M)}\right)^{-1}e,\qquad
    \beta:=a^\top\!\left(H_{\gamma,L}^{(M)}\right)^{-1}a.
\end{equation}
When \(\alpha>0\), choose
\begin{equation}
\label{eq:scalar-reduction-normalization}
    \lambda:=\frac1\alpha,\qquad
    \tau:=\frac1\beta,\qquad
    \kappa:=\lambda^2\rho .
\end{equation}
For \(r\ge1\), define
\begin{equation}
\label{eq:normalized-hard-instance}
    \bar f_{M,N,\gamma,L,a,r}
    :=
    f_{M,N,\gamma,L,a,\lambda,\kappa,\tau,r},
\end{equation}
where \(\lambda,\tau,\kappa\) are given by
\eqref{eq:scalar-reduction-normalization}.
\end{definition}

\begin{remark}[Entry-to-suffix coupling]
The term \(-\lambda u_{s-1}x_s\) is the structural replacement for the scalar
transmission term \(-u_{s-1}u_s\). If one coupled successive readouts directly
after the substitution \(u_s=a^\top z_s\), the gradient would point immediately
in the suffix direction \(a\), allowing a zero-respecting method to reach the
readout coordinates without traversing the block. We instead inject the previous
scalar state through the entry coordinate \(x_s=e_1^\top z_s\). The quadratic
block then propagates this signal toward the suffix readout, so the scalar chain
is lifted to a block chain while preserving the layer-by-layer revelation
structure. 
\end{remark}

\begin{remark}[Normalization]
The normalization in \eqref{eq:scalar-reduction-normalization} is chosen so
that the high-dimensional block chain reduces, at the level of the scalar
readouts \(u_s=a^\top z_s\), to the same scale as the scalar chain from
Proposition~\ref{prop:scalar-transition-lower-bound}. The quantities
\(\rho,\alpha,\beta\) are the Green-kernel responses of the quadratic block:
\(\alpha\) measures how an input through the entry coordinate reaches the suffix
readout, \(\beta\) measures the self-response of that readout, and \(\rho\)
measures the self-response of the entry coordinate. Thus
\(\lambda=1/\alpha\) normalizes the transmission from \(u_{s-1}\) to \(u_s\),
\(\tau=1/\beta\) puts the nonconvex potential on the same scalar scale, and
\(\kappa=\lambda^2\rho\) compensates for the curvature introduced by injecting
through the entry coordinate. In this sense, \(\bar f_{M,N,\gamma,L,a,r}\) is
the block-chain member whose scalar projection has the Carmon-chain
normalization.
\end{remark}

We now collect the basic estimates for the block chain that will be used later
in the lower-bound calibration.

\begin{lemma}[Boundedness and smoothness of normalized hard instances]\label{thm:boundedness-smoothness}
Fix \(p\ge2\), \(M,N\in\mathbb N\), \(\gamma,L>0\), \(r\ge1\), and a nonzero
vector \(a\in\mathbb R^M\). Assume \(\alpha>0\). Then the normalized
hard-instance family from \eqref{eq:normalized-hard-instance} satisfies
\[
    \operatorname{Lip}\!\left(\nabla \bar f_{M,N,\gamma,L,a,r}\right)
    \le
    \gamma+4L
    +
    \frac{2\|a\|}{\alpha}
    +
    \left(
        \frac{\rho}{\alpha^2}
        +
        \frac{\ell_1r^2}{\beta}
    \right)\|a\|^2 .
\]
For \(q=2,\ldots,p\),
\[
    \operatorname{Lip}\!\left(D^q\bar f_{M,N,\gamma,L,a,r}\right)
    \le
    \frac{\ell_qr^{3-q}\|a\|^{q+1}}{\beta}.
\]
Moreover,
\[
    \bar f_{M,N,\gamma,L,a,r}(0)
    -
    \inf \bar f_{M,N,\gamma,L,a,r}
    \le
    \frac{10N}{\beta}
    +
    \frac{\rho}{2\alpha^2}.
\]
\end{lemma}

\begin{proof}
See Appendix~\ref{subsec:proof-boundedness-smoothness}.
\end{proof}

\begin{remark}[Separating smoothness scales]
Lemma~\ref{thm:boundedness-smoothness} makes explicit the calibration role of
the vector \(a\). The first-order Lipschitz bound has a baseline scale
\(\gamma+4L\), so its dominant contribution can be fixed through \(L\). In
contrast, for \(q\ge2\), the \(q\)-th order Lipschitz bound depends on \(a\)
through \(\|a\|^{q+1}\) (with the normalization by \(\beta\)). Thus tuning
\(\|a\|\) separates the higher-order smoothness constraints from the
first-order scale.
\end{remark}

We next record the block-layer consequence needed in the lower-bound argument:
the last scalar chain variable remains hidden until the suffix of the last block
has been reached.

\begin{lemma}[Tail hiding for zero-respecting sequences]\label{thm:first-order-zero-chain}
Let \(F=f_{M,N,\gamma,L,a,\lambda,\kappa,\tau,r}\) be the block-chain function in Definition~\ref{def:function-class}, and let \(j_\star\in[M]\) satisfy
\[
\operatorname{supp}(a)\subseteq S:=\{j_\star,\ldots,M\}.
\]
For \(\ell\ge0\), define the block-layer coordinate set
\[
I_\ell
:=
\left\{
(s,j)\in[N]\times[M]:
(s-1)j_\star+\min\{j,j_\star\}\le \ell
\right\}.
\]
Let \(\{z^{(t)}\}_{t\ge0}\), \(z^{(t)}=(z_1^{(t)},\ldots,z_N^{(t)})\), be first-order zero-respecting with respect to \(F\), in the standard coordinate-support sense recalled in Section~\ref{sec:preliminaries}, with \(z^{(0)}=0\). Then, identifying supports with subsets of \([N]\times[M]\),
\[
\operatorname{supp}(z^{(t)})\subseteq I_t
\qquad\text{for all }t\ge0.
\]
Consequently, if \(t<Nj_\star\), then
\[
u_N^{(t)}:=a^\top z_N^{(t)}=0.
\]
\end{lemma}

\begin{proof}
See Appendix~\ref{subsec:proof-first-order-zero-chain}.
\end{proof}

\begin{lemma}[Gradient lower bound for normalized hard instances]\label{thm:gradient-lower-bound}
Fix \(M,N\in\mathbb N\), \(\gamma,L>0\), \(r\ge1\), and a nonzero vector
\(a\in\mathbb R^M\). Assume \(\alpha>0\). For any
\(z=(z_1,\ldots,z_N)\in(\mathbb R^M)^N\) satisfying
\[
    a^\top z_N=0,
\]
the normalized hard-instance family from \eqref{eq:normalized-hard-instance}
satisfies
\[
\left\|\nabla \bar f_{M,N,\gamma,L,a,r}(z)\right\|
    \ge
    \frac{1}{4}\frac{\sqrt{\gamma}}{
        \sqrt{\beta}
        +
        \frac{\beta\sqrt{\rho}}{\alpha}
    }.
\]
\end{lemma}

\begin{proof}
See Appendix~\ref{subsec:proof-gradient-lower-bound}.
\end{proof}

Lemmas~\ref{thm:boundedness-smoothness},
\ref{thm:first-order-zero-chain}, and
\ref{thm:gradient-lower-bound} are the uncalibrated interface of the
construction: the block geometry enters only through
\(\rho,\alpha,\beta\), \(\|a\|\), and the suffix layer \(j_\star\). It remains
to choose these quantities so that the smoothness budgets, initial gap, reveal
depth, and gradient certificate all have the desired scales.

\section{Main Results}\label{sec:theorems}
We now carry out the calibration left open by the block-chain construction.
Throughout this section, let \(p\ge2\), let
\(L_1,\ldots,L_p,\epsilon,\Delta>0\), and let
\(\ell_1,\ldots,\ell_p\) be the constants from
Proposition~\ref{prop:primitive-bounds}.
Fix, once and for all, the following numerical constants:
\begin{equation}\label{eq:explicit-calibration-constants}
\begin{gathered}
C_W:=512,\qquad c_{\rm hp}:=\frac1{50000},\qquad
C_M:=\frac1{\sqrt8},\qquad c_N:=\frac1{720C_W^2},\\
c_\epsilon:=\frac{c_{\rm hp}}{5184C_W^2},\qquad
c_T:=\frac{c_N\sqrt{c_{\rm hp}}}{6\sqrt8}.
\end{gathered}
\end{equation}

The following scale summarizes the active smoothness budget. For a fixed
potential parameter \(r\), \(\Gamma_p(r)\) is the largest mass scale compatible
with all derivative Lipschitz constraints, up to the numerical constants already
absorbed into the construction. We then optimize over \(r\).
Define
\[
\begin{aligned}
\Gamma_p(r)
&:=
\min\left\{
    L_1,\,
    \frac{L_1}{\ell_1r^2},\,
    \left(\frac{L_2\epsilon}{\ell_2r}\right)^{1/2},\,
    \min_{3\le q\le p}
    \left(
        \frac{L_q\epsilon^{q-1}r^{q-3}}{\ell_q}
    \right)^{1/q}
\right\},
\qquad r\ge1,\\
\Gamma_p
&:=
\sup_{r\ge1}\Gamma_p(r),
\end{aligned}
\]
where the last minimum is omitted when \(p=2\). The continuous positive
function \(\Gamma_p(r)\) tends to zero as \(r\to\infty\), because it is at
most \(L_1/(\ell_1r^2)\). Choose an actual maximizer
\(r_\star\ge1\), so that \(\Gamma_p(r_\star)=\Gamma_p\).

The parameter \(\Gamma_p\) is the only implicit quantity in the final rate. The
next proposition unpacks its \(\epsilon\)-dependence in the high-accuracy
regime.
\begin{proposition}[Asymptotic form of \(\Gamma_p\)]
\label{prop:gamma-p-upper-bound}
The following simplifications have explicit sufficient conditions. If
\(p=2\) and
\(\epsilon\le (L_1/\ell_1)^2\ell_2/L_2\), or if \(p=3\) and
\[
\epsilon\le
\min\left\{
\left(\frac{(L_1/\ell_1)^3\ell_3}{L_3}\right)^{1/2},
\left(\frac{L_2}{\ell_2}\right)^3
\left(\frac{\ell_3}{L_3}\right)^2
\right\},
\]
then
\[
\Gamma_p
=
\left(\frac{L_p\epsilon^{p-1}}{\ell_p}\right)^{1/p}.
\]
If \(p\ge4\), let \(\mathfrak G_p\) denote the supremum below and choose
\(R_\star>0\) attaining it. When
\(\epsilon\le\min\{R_\star^3,(L_1/\mathfrak G_p)^{3/2}\}\),
\[
\Gamma_p
=
\epsilon^{2/3}
\mathfrak G_p,\qquad
\mathfrak G_p:=\sup_{R>0}
\min\left\{
    \frac{L_1}{\ell_1R^2},\,
    \left(\frac{L_2}{\ell_2R}\right)^{1/2},\,
    \min_{3\le q\le p}
    \left(\frac{L_qR^{q-3}}{\ell_q}\right)^{1/q}
\right\}.
\]
\end{proposition}

\begin{proof}
See Appendix~\ref{subsec:proof-gamma-p-upper-bound}.
\end{proof}

We can now state the deterministic first-order lower bound.
\begin{theorem}[Main deterministic first-order lower bound]
\label{thm:first-order-lower-bound}
Let \(\mathcal A_{\rm det}\) denote the class of deterministic first-order
methods. Assume that
\(\epsilon\le c_\epsilon\min_{1\le q\le p}
\Delta^{q/(q+1)}(L_q/\ell_q)^{1/(q+1)}\), with the explicit
\(c_\epsilon\) in \eqref{eq:explicit-calibration-constants}. Then
\[
T_\epsilon\bigl(\mathcal A_{\rm det},
\mathcal F_{1:p}(\Delta,L_1,\ldots,L_p)\bigr)
\ge
c_T\Delta L_1^{1/2}\Gamma_p^{1/2}\epsilon^{-2}.
\]
If \(\epsilon\) also meets the corresponding sufficient condition of
Proposition~\ref{prop:gamma-p-upper-bound}, this implies
\[
T_\epsilon\bigl(\mathcal A_{\rm det},
\mathcal F_{1:p}(\Delta,L_1,\ldots,L_p)\bigr)
\ge c_T
\begin{cases}
\Delta L_1^{1/2}(L_2/\ell_2)^{1/4}\epsilon^{-7/4},
    & p=2,\\
\Delta L_1^{1/2}(L_3/\ell_3)^{1/6}\epsilon^{-5/3},
    & p=3,\\
\Delta L_1^{1/2}
\min_{1\le q\le p}(L_q/\ell_q)^{1/(2q)}\epsilon^{-5/3},
    & p\ge4.
\end{cases}
\]
\end{theorem}

We record three consequences of the theorem.
\begin{remark}[Structure of the lower bound]
The lower bound can be written as
\[
    \Delta L_1^{1/2}\Gamma_p^{1/2}\epsilon^{-2}
    =
    \Delta L_1\epsilon^{-2}
    \left(\frac{\Gamma_p}{L_1}\right)^{1/2}.
\]
Here \(\Delta L_1\epsilon^{-2}\) is the classical first-order complexity scale
under Lipschitz gradients alone. The remaining factor
\((\Gamma_p/L_1)^{1/2}\) is the acceleration factor imposed by the active
higher-order smoothness constraints. In this sense, \(\Gamma_p\) measures how
much of the baseline \(L_1\)-scale remains available after enforcing the
higher-order smoothness budgets; when the active constraint makes
\(\Gamma_p\) depend on \(\epsilon\), this factor converts the baseline
\(\epsilon^{-2}\) rate into the accelerated exponents.
\end{remark}

\begin{remark}[Matching upper bounds]
The special cases in Theorem~\ref{thm:first-order-lower-bound} match the known
upper bounds not only in the exponent of \(\epsilon\), but also in their
dependence on the initial gap and smoothness constants. In the
Hessian-Lipschitz regime, the \(\epsilon^{-7/4}\) rate appears in several
accelerated or curvature-exploiting algorithms
\citep{carmon2018acceleratedmethods,carmon2017convexuntilproven,agarwal2017finding,jin2018acceleratedsaddle,zhang2021escapesaddle,li2023restartedagd,marumo2024parameterfreeagd}.
In the purely deterministic first-order stationarity setting, these bounds have
the scale \(O(\Delta L_1^{1/2}L_2^{1/4}\epsilon^{-7/4})\), up to logarithmic
factors in some variants; the restarted methods of \citet{li2023restartedagd}
remove this logarithmic loss. For third-order
smoothness, \citet{carmon2017convexuntilproven} obtain the corresponding
first-order upper bound
\(O(\Delta L_1^{1/2}L_3^{1/6}\epsilon^{-5/3}
\log(L_1\Delta/\epsilon^2))\) gradient evaluations. Thus the \(p=2\) and
\(p=3\) lower bounds match the polynomial dependence on \(\epsilon\),
\(\Delta\), and the active smoothness constants in the corresponding upper
bounds. The parameter-free method of \citet{xiong2026parameterfree} also has
an \(O(\epsilon^{-5/3}\log(1/\epsilon))\) gradient-evaluation bound when
problem-dependent factors are suppressed. Its polynomial \(\epsilon\)-exponent
therefore agrees with our \(p=3\) lower bound, but this comparison neither
removes the logarithmic factor nor asserts that its initialization-dependent
prefactor matches ours.
\end{remark}

\begin{remark}[Saturation of higher-order acceleration]
Together with the third-order first-order upper bound, the last line of
Theorem~\ref{thm:first-order-lower-bound} shows that the
\(\epsilon\)-dependence saturates once third derivatives are Lipschitz. Indeed,
for every \(p\ge4\), adding Lipschitz bounds on derivatives of order \(4\) and
higher does not improve the dimension-free first-order exponent beyond
\(\epsilon^{-5/3}\). Higher-order smoothness may still matter through the
dependence on the smoothness constants: the active term in \(\Gamma_p\), or in
the simplified bound above, can change when the higher-order Lipschitz budgets
are favorable. This gives lower-bound support for the
corresponding observation in the ``convex until proven guilty'' framework of
\citet{carmon2017convexuntilproven}.
\end{remark}

It remains to specify the calibrated instance used in the proof of
Theorem~\ref{thm:first-order-lower-bound}. Lemmas~\ref{thm:boundedness-smoothness},
\ref{thm:first-order-zero-chain}, and \ref{thm:gradient-lower-bound} reduce this
to choosing the block geometry:
the block length creates reveal depth, the number of phases uses the available
gap, and the suffix vector \(a\) sets the low-rank smoothness and gradient
certificate scales.

\begin{condition}[Calibrated parameters]\label{cond:hyperparameter-choice}
Assume the tolerance condition of Theorem~\ref{thm:first-order-lower-bound}
and use the constants in \eqref{eq:explicit-calibration-constants}. Set
\[
\begin{gathered}
\gamma:=c_{\rm hp}\Gamma_p,\qquad L:=L_1/8,\qquad
h:=\sqrt{\gamma/L},\qquad
M:=\lceil h^{-1}\rceil
=\left\lceil C_M\sqrt{L_1/\gamma}\right\rceil,\\
j_\star:=\lfloor M/2\rfloor,\qquad
S:=\{j_\star,\ldots,M\},\qquad
W:=\frac{C_W\epsilon}{\gamma\sqrt h},\qquad
a:=\frac{\mathbf 1_S}{W},\\
N:=\left\lfloor c_N\frac{\Delta\gamma}{\epsilon^2}\right\rfloor,
\qquad r:=r_\star.
\end{gathered}
\]
The tolerance condition implies \(h<1/4\), \(M\ge4\), and \(N\ge7\), as
verified in the proof below. In particular, the suffix and every block are
well defined.
Let \(H:=H_{\gamma,L}^{(M)}\). For this choice of \(a\), the positivity
condition required by the scalar-reduction normalization holds; define
\(\lambda,\tau,\kappa\) by \eqref{eq:scalar-reduction-normalization}.
Set \(d:=MN\) and, under the block
identification \(\mathbb R^d\cong(\mathbb R^M)^N\), define the calibrated
hard instance
\[
    f_{\Delta,L_{1:p},\epsilon}
    :=
    \bar f_{M,N,\gamma,L,a,r}.
\]
\end{condition}

The following lemma verifies that this choice realizes the desired scales and
therefore produces a hard instance in the target smoothness class.
\begin{lemma}[Zero-respecting lower bound for the calibrated instance]
\label{thm:calibrated-zero-respecting}
Let the parameters be chosen as in Condition~\ref{cond:hyperparameter-choice}.
Then the quantities from the scalar-reduction normalization satisfy
\[
\begin{gathered}
\frac1{9Lh}\le\rho\le\frac{16}{Lh},\qquad
\frac{\sqrt h}{18C_W\epsilon}\le\alpha\le\frac{\sqrt h}{C_W\epsilon},\\
\frac{\gamma}{36C_W^2\epsilon^2}\le\beta
\le\frac{\gamma}{C_W^2\epsilon^2},\qquad
\frac{\gamma^2}{2C_W^2\epsilon^2}\le\|a\|^2
\le\frac{\gamma^2}{C_W^2\epsilon^2}.
\end{gathered}
\]
The calibrated instance \(f_{\Delta,L_{1:p},\epsilon}\) belongs to
\(\mathcal F_{1:p}(\Delta,L_1,\ldots,L_p)\). Moreover, any first-order
zero-respecting method \(\mathsf A\) applied to
\(f_{\Delta,L_{1:p},\epsilon}\), starting
from the origin, satisfies
\[
    T_\epsilon\bigl(\mathsf A,f_{\Delta,L_{1:p},\epsilon}\bigr)
    \ge c_T\Delta L_1^{1/2}\Gamma_p^{1/2}\epsilon^{-2}.
\]
\end{lemma}

\begin{proof}
See Appendix~\ref{subsec:proof-calibrated-zero-respecting}.
\end{proof}

\begin{proof}[Proof of Theorem~\ref{thm:first-order-lower-bound}]
Lemma~\ref{thm:calibrated-zero-respecting} gives the stated lower bound for
zero-respecting methods on an instance in
\(\mathcal F_{1:p}(\Delta,L_1,\ldots,L_p)\). The class
\(\mathcal F_{1:p}(\Delta,L_1,\ldots,L_p)\) is closed under rectangular
orthogonal embeddings: if \(U^\top U=I_d\), then
\(f_U(x):=f(U^\top x)\) has the same initial gap and no larger derivative
Lipschitz constants. Thus Proposition~1 of
\citet{carmon2017lowerfirstorder} transfers
the zero-respecting lower bound to arbitrary deterministic first-order methods.
Taking the infimum over \(\mathsf A\in\mathcal A_{\rm det}\) gives the minimax
claim. The explicit display follows by substituting
Proposition~\ref{prop:gamma-p-upper-bound}; in the \(p\ge4\) regime we use its
supremum at \(R=1\), which gives
\(\Gamma_p\ge\epsilon^{2/3}\min_{1\le q\le p}(L_q/\ell_q)^{1/q}\).
\end{proof}

\section{Conclusion}\label{sec:conclusion}

We established a dimension-free deterministic first-order lower bound for
finding approximate stationary points under higher-order smoothness. In the
Hessian-Lipschitz case the bound gives the matching
\(\Omega(\epsilon^{-7/4})\) dependence, and under Lipschitz third derivatives it
gives the matching \(\Omega(\epsilon^{-5/3})\) dependence. Thus the accelerated
rates achieved by first-order methods under higher-order smoothness are optimal,
up to the standard logarithmic factors present in some upper bounds. The main technical device is the block-chain hard instance. It separates the
scalar transition obstruction from the blockwise oracle-discovery process: the
scalar state is carried by a suffix readout inside each block, while the path
quadratic forces zero-respecting methods to expose the block one coordinate
layer at a time. 

\section*{Acknowledgements}

DZ first began thinking about this problem in 2019, but did not make
substantial progress at that time. Motivated by the recent advances in the
mathematical reasoning capabilities of AI models, DZ revisited the problem in
May 2026 with AI assistance. DZ supplied the relevant prior literature and
designed prompts asking ChatGPT 5.5 Pro to search for a first-order lower-bound
construction. Through several rounds of interaction, ChatGPT 5.5 Pro identified
the block-chain construction as the right direction and began validating
its correctness. DZ then took primary responsibility for checking the proofs,
simplifying several proof paths, improving the presentation, and formulating
extension directions, again with assistance from ChatGPT 5.5 Pro.

We thank Oliver Hinder and Yair Carmon for suggestions on making the constants
explicit and on potential numerical experiments. We thank Jiaxiang ``Jason'' Li
and Coralia Cartis for discussions of related work.

\appendix

\begin{center}
{\LARGE\bfseries Appendix}
\end{center}
\vspace{1em}

\section{Proofs for the Block-Chain Construction}\label{app:proofs-section4}

\subsection{Proof of Lemma~\ref{thm:boundedness-smoothness}}
\label{subsec:proof-boundedness-smoothness}

\begin{proof}[Proof of Lemma~\ref{thm:boundedness-smoothness}]
Let \(\Upsilon:=\Upsilon_r\) and
\(F:=\bar f_{M,N,\gamma,L,a,r}\), and let
\(e,\rho,\alpha,\beta,\lambda,\tau,\kappa\) be the quantities in
\eqref{eq:scalar-reduction-geometry} and
\eqref{eq:scalar-reduction-normalization} for the present
\((M,\gamma,L,a)\), and set \(H:=H_{\gamma,L}^{(M)}\). In particular,
\(\kappa=\rho/\alpha^2\). Write
\(z=(z_1,\ldots,z_N)\in(\mathbb R^M)^N\), \(u_s:=a^\top z_s\), and
\(u_0:=1\).
We use the uniform properties
\(m:=\inf_x\Upsilon(x)=\Upsilon(1)=0\), \(\Upsilon(0)\le 10\), and, for
\(q=1,\ldots,p\), \(\|\Upsilon^{(q+1)}\|_\infty\le \ell_q r^{3-q}\).

\paragraph{Smoothness.}
We first prove the \(q=1\) smoothness bound. Decompose \(F=Q+R\), where
\[
    Q(z)
    :=
    \sum_{s=1}^N
    \left[
        \frac12 z_s^\top H z_s
        -
        \lambda u_{s-1}e^\top z_s
        +
        \frac{\kappa}{2}u_s^2
    \right],
\]
and \(R(z):=\tau\sum_{s=1}^N\Upsilon(u_s)\).
For any block vector \(v=(v_1,\ldots,v_N)\), the quadratic part satisfies
\[
\begin{aligned}
    D^2Q(z)[v,v]
    &=
    \sum_{s=1}^N v_s^\top H v_s
    +
    \kappa\sum_{s=1}^N(a^\top v_s)^2
    -
    2\lambda\sum_{s=2}^N(a^\top v_{s-1})(e^\top v_s).
\end{aligned}
\]
Hence, using \(\|e\|=1\),
\[
\begin{aligned}
    |D^2Q(z)[v,v]|
    &\le
    \|H\|\sum_{s=1}^N\|v_s\|^2
    +
    \kappa\|a\|^2\sum_{s=1}^N\|v_s\|^2
    +
    2\lambda\|a\|\sum_{s=2}^N\|v_{s-1}\|\|v_s\|
    \\
    &\le
    \left(
        \|H\|+\kappa\|a\|^2+2\lambda\|a\|
    \right)\|v\|^2.
\end{aligned}
\]
Therefore
\[
    \|D^2Q(z)\|_{\rm op}
    \le
    \gamma+4L
    +
    \frac{\rho}{\alpha^2}\|a\|^2
    +
    \frac{2\|a\|}{\alpha}.
\]

For the nonlinear part,
\[
    D^2R(z)[v,v]
    =
    \tau\sum_{s=1}^N
    \Upsilon''(u_s)(a^\top v_s)^2.
\]
Thus
\[
\begin{aligned}
    |D^2R(z)[v,v]|
    &\le
    \tau\|\Upsilon''\|_\infty
    \sum_{s=1}^N |a^\top v_s|^2
    \\
    &\le
    \frac{\ell_1 r^2}{\beta}\|a\|^2\|v\|^2.
\end{aligned}
\]
Consequently,
\[
    \|D^2F(z)\|_{\rm op}
    \le
    \gamma+4L
    +
    \frac{2\|a\|}{\alpha}
    +
    \left(
        \frac{\rho}{\alpha^2}
        +
        \frac{\ell_1r^2}{\beta}
    \right)\|a\|^2
    \qquad
    \text{for every }z,
\]
and hence
\[
    \operatorname{Lip}(\nabla F)
    \le
    \gamma+4L
    +
    \frac{2\|a\|}{\alpha}
    +
    \left(
        \frac{\rho}{\alpha^2}
        +
        \frac{\ell_1r^2}{\beta}
    \right)\|a\|^2 .
\]

We now prove the higher-order smoothness bounds. Fix \(q\in\{2,\ldots,p\}\).
All terms in \(Q\) are quadratic, so \(D^{q+1}Q\equiv0\). Thus only \(R\) contributes.
For arbitrary block vectors
\[
    v^{(0)},v^{(1)},\ldots,v^{(q)}\in(\mathbb R^M)^N,
    \qquad
    v^{(\ell)}=(v^{(\ell)}_1,\ldots,v^{(\ell)}_N),
\]
we have
\[
\begin{aligned}
    D^{q+1}R(z)
    [v^{(0)},v^{(1)},\ldots,v^{(q)}]
    &=
    \tau
    \sum_{s=1}^N
    \Upsilon^{(q+1)}(u_s)
    \prod_{\ell=0}^{q}
    \bigl(a^\top v_s^{(\ell)}\bigr).
\end{aligned}
\]
Therefore
\[
\begin{aligned}
    \left|
    D^{q+1}R(z)
    [v^{(0)},v^{(1)},\ldots,v^{(q)}]
    \right|
    &\le
    \frac{\ell_q r^{3-q}}{\beta}
    \sum_{s=1}^N
    \prod_{\ell=0}^{q}
    \left|a^\top v_s^{(\ell)}\right|.
\end{aligned}
\]
We use the elementary inequality that for nonnegative sequences
\(b^{(0)},\ldots,b^{(q)}\),
\[
    \sum_s\prod_{\ell=0}^q b_s^{(\ell)}
    \le
    \prod_{\ell=0}^q
    \left(\sum_s (b_s^{(\ell)})^2\right)^{1/2}.
\]
Applying this with \(b_s^{(\ell)}:=|a^\top v_s^{(\ell)}|\) gives
\[
\begin{aligned}
    \sum_{s=1}^N
    \prod_{\ell=0}^{q}
    |a^\top v_s^{(\ell)}|
    &\le
    \prod_{\ell=0}^{q}
    \left(
        \sum_{s=1}^N |a^\top v_s^{(\ell)}|^2
    \right)^{1/2}
    \\
    &\le
    \|a\|^{q+1}
    \prod_{\ell=0}^{q}
    \|v^{(\ell)}\|.
\end{aligned}
\]
Hence
\[
    \|D^{q+1}F(z)\|_{\rm op}
    =
    \|D^{q+1}R(z)\|_{\rm op}
    \le
    \frac{\ell_q r^{3-q}\|a\|^{q+1}}{\beta}.
\]
By the fundamental theorem of calculus along line segments,
\[
    \|D^qF(x)-D^qF(y)\|_{\rm op}
    \le
    \sup_{0\le t\le1}
    \|D^{q+1}F(y+t(x-y))\|_{\rm op}
    \|x-y\|,
\]
so
\[
    \operatorname{Lip}(D^qF)
    \le
    \frac{\ell_q r^{3-q}\|a\|^{q+1}}{\beta},
    \qquad
    q=2,\ldots,p.
\]

\paragraph{Gap bound.}
For each phase \(s\), write \(v:=u_{s-1}\). Since \(H\succ0\),
\[
\begin{aligned}
\frac12 z_s^\top H z_s-\lambda v e^\top z_s
&=
\frac12
\left\|z_s-\lambda vH^{-1}e\right\|_H^2
-\frac{\lambda^2v^2}{2}e^\top H^{-1}e\\
&\ge
-\frac{\kappa}{2}v^2,
\end{aligned}
\]
where \(\|w\|_H^2:=w^\top Hw\), \(\kappa=\lambda^2\rho\), and
\(\rho=e^\top H^{-1}e\). Hence
\[
\begin{aligned}
&\frac12 z_s^\top H z_s-\lambda u_{s-1}e^\top z_s
+\frac{\kappa}{2}u_s^2+\tau\Upsilon(u_s)\\
&\qquad\ge
\frac{\kappa}{2}(u_s^2-u_{s-1}^2)+\tau\Upsilon(u_s).
\end{aligned}
\]
Summing over \(s=1,\ldots,N\), using \(u_0=1\), \(u_N^2\ge0\), and
\(\Upsilon\ge0\), gives
\(F(z)\ge \frac{\kappa}{2}(u_N^2-u_0^2)\ge-\frac{\kappa}{2}\).
At the origin, \(u_s=0\) for every \(s\ge1\), and all quadratic, linear, and
stabilization terms vanish. Since \(\Upsilon(0)\le10\),
\(F(0)\le10N\tau\).
Therefore
\[
F(0)-\inf_zF(z)
\le
10N\tau+\frac{\kappa}{2}
=
\frac{10N}{\beta}+\frac{\rho}{2\alpha^2}.
\]
\end{proof}

\subsection{Proof of Lemma~\ref{thm:first-order-zero-chain}}\label{subsec:proof-first-order-zero-chain}

\begin{proof}[Proof of Lemma~\ref{thm:first-order-zero-chain}]
For each phase \(s\), all coordinates in \(\{s\}\times S\) have layer
index \(s j_\star\), since for \(j\in S\),
\((s-1)j_\star+\min\{j,j_\star\}=s j_\star\). We prove by induction that
\(\operatorname{supp}(z^{(t)})\subseteq I_t\) for every \(t\ge0\),
where supports are understood in block-coordinate notation.

The key step is the following one-step support implication. Suppose \(\operatorname{supp}(z)\subseteq I_t\). Write \(H=H_{\gamma,L}^{(M)}\), \(u_s=a^\top z_s\), \(x_s=e_1^\top z_s\), and use the convention \(x_{N+1}=0\). The \(s\)-th gradient block is
\[
g_s
=
Hz_s
-
\lambda u_{s-1}e_1
+
(\kappa u_s+\tau\Upsilon_r'(u_s))a
-
\lambda x_{s+1}a.
\]
We check that each term is supported in \(I_{t+1}\).

Since \(H\) is tridiagonal, \(Hz_s\) can only move support inside phase \(s\) by one neighboring coordinate. In terms of \(I_t\), this can increase \(\min\{j,j_\star\}\) by at most one, so \(Hz_s\) is supported in \(I_{t+1}\).

For the entry term \(-\lambda u_{s-1}e_1\), if \(s=1\), then the coordinate \((1,1)\) belongs to \(I_{t+1}\). If \(s\ge2\) and \(u_{s-1}\ne0\), then \(z_{s-1}\) has some support in the suffix \(S\), so \((s-1)j_\star\le t\). Hence the coordinate \((s,1)\), whose layer index is \((s-1)j_\star+1\), belongs to \(I_{t+1}\).

For the term \((\kappa u_s+\tau\Upsilon_r'(u_s))a\), if \(z_s\) has no support in \(S\), then \(u_s=0\), and Proposition~\ref{prop:primitive-bounds} gives \(\Upsilon_r'(0)=0\), so the coefficient vanishes. Otherwise \(\{s\}\times S\) is already contained in \(I_t\), and the term is supported in \(I_t\subseteq I_{t+1}\).

Finally, consider \(-\lambda x_{s+1}a\). If \(x_{s+1}=0\), the term vanishes. If \(x_{s+1}\ne0\), then coordinate \((s+1,1)\) belongs to the support of \(z\), so \(sj_\star+1\le t\). Thus \(\{s\}\times S\), whose coordinates all have layer index \(sj_\star\), is contained in \(I_t\). Hence this term is also supported in \(I_t\subseteq I_{t+1}\). Therefore \(\operatorname{supp}(\nabla F(z))\subseteq I_{t+1}\) whenever \(\operatorname{supp}(z)\subseteq I_t\).

Now use the zero-respecting property. Since \(z^{(0)}=0\), the induction starts with \(\operatorname{supp}(z^{(0)})\subseteq I_0\). If \(\operatorname{supp}(z^{(t')})\subseteq I_{t'}\) for all \(t'<t\), then the one-step implication gives
\(\operatorname{supp}(\nabla F(z^{(t')}))\subseteq I_{t'+1}\subseteq I_t\)
for \(t'<t\).
By zero-respectingness,
\[
\operatorname{supp}(z^{(t)})
\subseteq
\bigcup_{t'<t}\operatorname{supp}(\nabla F(z^{(t')}))
\subseteq I_t.
\]
This proves the induction claim.

If \(t<Nj_\star\), then no coordinate \((N,j)\) with \(j\in S\) lies in \(I_t\). Since \(\operatorname{supp}(a)\subseteq S\), we conclude \(u_N^{(t)}=a^\top z_N^{(t)}=0\).
\end{proof}

\subsection{Proof of Lemma~\ref{thm:gradient-lower-bound}}\label{subsec:proof-gradient-lower-bound}

\begin{lemma}[Gradient projection identity]\label{lem:gradient-projection}
For \(M,N\in\mathbb N\), \(\gamma,L,r>0\), and any nonzero
vector \(a\in\mathbb R^M\), let
\(e,\rho,\alpha,\beta,\lambda,\tau,\kappa\) be the quantities in
\eqref{eq:scalar-reduction-geometry} and
\eqref{eq:scalar-reduction-normalization} for the present
\((M,\gamma,L,a)\), and set \(H:=H_{\gamma,L}^{(M)}\). Let \(g_s\) denote the \(s\)-th block of
\(\nabla\bar f_{M,N,\gamma,L,a,r}(z)\), namely
\[
g_s
=
Hz_s
-
\lambda u_{s-1}e
+
(\kappa u_s+\tau\Upsilon_r'(u_s))a
-
\lambda(e^\top z_{s+1})a,
\]
where \(u_s=a^\top z_s\). With
\[
w_{\mathrm{out}}=\frac{H^{-1}a}{\beta},\qquad
w_{\mathrm{in}}=H^{-1}\left(e-\frac{\alpha}{\beta}a\right),
\]
one has, for \(s=1,\ldots,N-1\),
\[
w_{\mathrm{out}}^\top g_s+\lambda w_{\mathrm{in}}^\top g_{s+1}
=
\tau\left[
2u_s-u_{s-1}-u_{s+1}+\Upsilon_r'(u_s)
\right].
\]
\end{lemma}

\begin{proof}
Write $x_s=e^\top z_s$ within this proof. We first expand the two projections without substituting the parameter identities $\lambda=1/\alpha$, $\tau=1/\beta$, and $\kappa=\lambda^2\rho$. Since
\[
w_{\mathrm{out}}=\frac{H^{-1}a}{\beta},
\qquad
w_{\mathrm{out}}^\top H z_s=\frac{u_s}{\beta},\qquad
w_{\mathrm{out}}^\top e=\frac{\alpha}{\beta},\qquad
w_{\mathrm{out}}^\top a=1,
\]
we have
\[
w_{\mathrm{out}}^\top g_s
=
\frac{u_s}{\beta}
-
\lambda\frac{\alpha}{\beta}u_{s-1}
+
\kappa u_s
+
\tau\Upsilon_r'(u_s)
-
\lambda x_{s+1}.
\]
For the second projection, use only \(a^\top w_{\mathrm{in}}=0\) to remove the terms parallel to \(a\):
\[
\begin{aligned}
w_{\mathrm{in}}^\top g_{s+1}
&=
w_{\mathrm{in}}^\top H z_{s+1}
-
\lambda u_s w_{\mathrm{in}}^\top e
+
(\kappa u_{s+1}+\tau\Upsilon_r'(u_{s+1}))a^\top w_{\mathrm{in}}
-
\lambda x_{s+2}a^\top w_{\mathrm{in}}
\\
&=
x_{s+1}
-
\frac{\alpha}{\beta}u_{s+1}
-
\lambda(w_{\mathrm{in}}^\top e)u_s .
\end{aligned}
\]
Thus \(\lambda w_{\mathrm{in}}^\top g_{s+1}
=\lambda x_{s+1}
-\lambda\frac{\alpha}{\beta}u_{s+1}
-\lambda^2(w_{\mathrm{in}}^\top e)u_s\).
Adding the two expansions cancels the $x_{s+1}$ terms and gives
\[
w_{\mathrm{out}}^\top g_s+\lambda w_{\mathrm{in}}^\top g_{s+1}
=
\left(
\frac1\beta+\kappa-\lambda^2 w_{\mathrm{in}}^\top e
\right)u_s
-
\lambda\frac{\alpha}{\beta}u_{s-1}
-
\lambda\frac{\alpha}{\beta}u_{s+1}
+
\tau\Upsilon_r'(u_s).
\]
Finally, \(w_{\mathrm{in}}^\top e=\rho-\frac{\alpha^2}{\beta}\).
Using the parameter choices $\lambda=1/\alpha$, $\tau=1/\beta$, and $\kappa=\lambda^2\rho$, the coefficient of $u_s$ becomes
\[
\frac1\beta+\lambda^2\rho-\lambda^2\left(\rho-\frac{\alpha^2}{\beta}\right)
=
\frac2\beta
=
2\tau,
\]
while \(\lambda\frac{\alpha}{\beta}=\frac1\beta=\tau\).
Therefore
\[
w_{\mathrm{out}}^\top g_s+\lambda w_{\mathrm{in}}^\top g_{s+1}
=
\tau\left[
2u_s-u_{s-1}-u_{s+1}+\Upsilon_r'(u_s)
\right].
\]
\end{proof}

We now prove Lemma~\ref{thm:gradient-lower-bound}.

\begin{proof}[Proof of Lemma~\ref{thm:gradient-lower-bound}]
Write \(u_s:=a^\top z_s\), \(u_0:=1\), and \(u_{N+1}:=0\), and let
\(g_s\) denote the \(s\)-th block of
\(\nabla \bar f_{M,N,\gamma,L,a,r}(z)\). Let
\(e,\rho,\alpha,\beta,\lambda,\tau,\kappa\) be the quantities in
\eqref{eq:scalar-reduction-geometry} and
\eqref{eq:scalar-reduction-normalization} for the present
\((M,\gamma,L,a)\), set \(H:=H_{\gamma,L}^{(M)}\), and define
\[
    w_{\mathrm{out}}:=\frac{H^{-1}a}{\beta},
    \qquad
    w_{\mathrm{in}}:=H^{-1}\left(e-\frac{\alpha}{\beta}a\right).
\]
Set \(r_s:=2u_s-u_{s-1}-u_{s+1}+\Upsilon_r'(u_s)\), \(s=1,\ldots,N\).
Since \(u_N=u_{N+1}=0\), Proposition~\ref{prop:scalar-transition-lower-bound}
with \(\mu=1\) gives \(\|\nabla\phi_1(u)\|_2\ge c_0\). The last gradient
coordinate of \(\phi_1\) is \(u_{N+1}-u_N=0\), so \(\|r\|_2\ge c_0\).
By Lemma~\ref{lem:gradient-projection}, for \(s=1,\ldots,N-1\),
\[
    \tau r_s
    =
    w_{\mathrm{out}}^\top g_s
    +
    \lambda w_{\mathrm{in}}^\top g_{s+1}.
\]
For the last phase, \(u_N=0\), \(u_{N+1}=0\), and
\(\Upsilon_r'(0)=0\), so the same projection expansion gives
\(\tau r_N=w_{\mathrm{out}}^\top g_N\). Let
\(A:=\|w_{\mathrm{out}}\|+\lambda\|w_{\mathrm{in}}\|\).
Regard \((g_s)_{s=1}^N\) as an \(\ell_2\) sequence and put
\(g_{N+1}=0\). The two projected sequences have norms at most
\(\|w_{\mathrm{out}}\|\|g\|_2\) and
\(\lambda\|w_{\mathrm{in}}\|\|g\|_2\), respectively; the latter is
shifted by one place. The triangle inequality therefore gives the
dimension-independent bound
\[
    \tau\|r\|_2\le A\|g\|_2.
\]
Since \(\|r\|_2\ge c_0=1/4\), we obtain
\(\|\nabla\bar f_{M,N,\gamma,L,a,r}(z)\|\ge\tau/(4A)\).
Finally, since \(H\succeq \gamma I\),
\[
    \|H^{-1}a\|^2
    =
    a^\top H^{-2}a
    \le
    \frac{1}{\gamma}a^\top H^{-1}a
    =
    \frac{\beta}{\gamma},
\]
and
\[
    \|w_{\mathrm{out}}\|
    =
    \frac{\|H^{-1}a\|}{\beta}
    \le
    \frac{1}{\sqrt{\gamma\beta}},
\]
and
\[
    \|w_{\mathrm{in}}\|^2
    \le
    \frac{1}{\gamma}
    \left(e-\frac{\alpha}{\beta}a\right)^\top
    H^{-1}
    \left(e-\frac{\alpha}{\beta}a\right)
    \le
    \frac{\rho}{\gamma}.
\]
The last inequality follows by expanding the quadratic form and using the
definitions of \(\alpha,\beta,\rho\).
Thus
\(\lambda\|w_{\mathrm{in}}\|\le \frac{1}{\alpha\sqrt{\gamma}}\sqrt{\rho}\).
Therefore
\[
    \frac{\tau}{A}
    \ge
    \frac{\sqrt{\gamma}}{
        \sqrt{\beta}
        +
        \frac{\beta\sqrt{\rho}}{\alpha}
    }.
\]
This proves the claim.
\end{proof}

\section{Proofs of the Main Results}\label{app:proofs-section5}

\subsection{Proof of Proposition~\ref{prop:gamma-p-upper-bound}}
\label{subsec:proof-gamma-p-upper-bound}

\begin{proof}[Proof of Proposition~\ref{prop:gamma-p-upper-bound}]
First consider \(p\in\{2,3\}\). In this case all terms in \(\Gamma_p(r)\) are
nonincreasing for \(r\ge1\), so the supremum is attained at \(r=1\). Thus
\[
\Gamma_p
=
\min\left\{
    L_1,\,
    \frac{L_1}{\ell_1},\,
    \min_{2\le q\le p}
    \left(\frac{L_q\epsilon^{q-1}}{\ell_q}\right)^{1/q}
\right\}.
\]
For \(p=2\), the stated threshold makes the \(q=2\) term no larger than
\(L_1/\ell_1\le L_1\). For \(p=3\), its two thresholds make the
\(q=3\) term no larger than \(L_1/\ell_1\) or the \(q=2\) term.
This proves the first claim.

Now suppose \(p\ge4\). For \(R>0\), define
\[
\mathfrak G_p(R)
:=
\min\left\{
    \frac{L_1}{\ell_1R^2},\,
    \left(\frac{L_2}{\ell_2R}\right)^{1/2},\,
    \min_{3\le q\le p}
    \left(\frac{L_qR^{q-3}}{\ell_q}\right)^{1/q}
\right\}.
\]
The function \(\mathfrak G_p(R)\) is continuous and positive on
\((0,\infty)\), and it is bounded above by the \(q=3\) term
\((L_3/\ell_3)^{1/3}\). Moreover, \(\mathfrak G_p(R)\to0\) as
\(R\to\infty\), and \(\mathfrak G_p(R)\to0\) as \(R\downarrow0\), because one
of the terms with \(q\ge4\) is present. Hence the supremum
\(\mathfrak G_p=\sup_{R>0}\mathfrak G_p(R)\) is finite, positive, and attained
at some \(R_\star>0\).

For \(r\ge1\), set \(R:=r\epsilon^{1/3}\). Equivalently,
\(r=R\epsilon^{-1/3}\), and the constraint \(r\ge1\) becomes
\(R\ge\epsilon^{1/3}\). With this change of variables, each nonconstant term in
\(\Gamma_p(r)\) has the same factor \(\epsilon^{2/3}\):
\[
\frac{L_1}{\ell_1r^2}
=
\epsilon^{2/3}\frac{L_1}{\ell_1R^2},
\qquad
\left(\frac{L_2\epsilon}{\ell_2r}\right)^{1/2}
=
\epsilon^{2/3}\left(\frac{L_2}{\ell_2R}\right)^{1/2},
\]
and, for \(q=3,\ldots,p\),
\[
\left(
    \frac{L_q\epsilon^{q-1}r^{q-3}}{\ell_q}
\right)^{1/q}
=
\epsilon^{2/3}
\left(
    \frac{L_qR^{q-3}}{\ell_q}
\right)^{1/q}.
\]
Therefore
\[
\Gamma_p
=
\sup_{R\ge\epsilon^{1/3}}
\min\left\{
    L_1,\,
    \epsilon^{2/3}\mathfrak G_p(R)
\right\}
\le
\epsilon^{2/3}\mathfrak G_p .
\]

It remains to prove the matching lower bound. Under the stated condition,
\(r_\star=R_\star\epsilon^{-1/3}\ge1\) and
\(\epsilon^{2/3}\mathfrak G_p\le L_1\). Therefore
\[
\Gamma_p
\ge
\Gamma_p(r_\star)
=
\min\left\{
    L_1,\,
    \epsilon^{2/3}\mathfrak G_p(R_\star)
\right\}
=
\epsilon^{2/3}\mathfrak G_p
\]
under the stated condition. Together with the upper bound, this proves
\(\Gamma_p=\epsilon^{2/3}\mathfrak G_p\).
\end{proof}

\subsection{Proof of Lemma~\ref{thm:calibrated-zero-respecting}}
\label{subsec:proof-calibrated-zero-respecting}

\begin{lemma}[Path Green-kernel estimate]\label{lem:green-kernel-hyperparameter-choice}
Let \(0<h\le1/4\), \(M=\lceil h^{-1}\rceil\), and let
\[
    A_h:=h^2 I_M+\mathsf P_M,
    \qquad
    H:=L A_h=\gamma I_M+L\mathsf P_M,
    \qquad
    h=\sqrt{\gamma/L}.
\]
Then, for all \(i,j\in[M]\),
\[
    \frac{e^{-h|i-j|}}{3Lh}
    \le
    (H^{-1})_{ij}
    \le
    \frac{16e^{-h|i-j|/2}}{Lh}.
\]
In particular, \(1/(9Lh)<(H^{-1})_{ij}\le16/(Lh)\).
\end{lemma}

\begin{proof}
Here \(M\ge4\). Write \(A=A_h\). In coordinates,
\[
A=
\begin{pmatrix}
1+h^2 & -1 \\
-1 & 2+h^2 & -1 \\
& \ddots & \ddots & \ddots \\
&& -1 & 2+h^2 & -1 \\
&&& -1 & 1+h^2
\end{pmatrix}.
\]

We use the standard inverse formula for tridiagonal matrices in the following
form. For a tridiagonal matrix with diagonal entries
\(\alpha_1,\ldots,\alpha_n\) and off-diagonal entries
\(-\beta_1,\ldots,-\beta_{n-1}\), define
\(\delta_1=\alpha_1\),
\(\delta_{k+1}=\alpha_{k+1}-\frac{\beta_k^2}{\delta_k}\),
\(d_n=\alpha_n\), and
\(d_{k-1}=\alpha_{k-1}-\frac{\beta_{k-1}^2}{d_k}\).
Then, with
\[
    b_i=\frac{\beta_1\cdots\beta_{i-1}}{d_1\cdots d_i},
    \qquad
    a_i=\frac{\beta_i\cdots\beta_{n-1}}{\delta_i\cdots\delta_n\, b_n},
\]
the inverse entries are \(a_{i\wedge j}b_{i\vee j}\). In our case,
\[
    \alpha_1=\alpha_M=1+h^2,\qquad
    \alpha_k=2+h^2\quad(2\le k\le M-1),
    \qquad
    \beta_k=1.
\]
Thus for \(i\le j\), using the convention that empty products equal \(1\),
\begin{equation*}
    (A^{-1})_{ij}
    =
    \frac{
    (\delta_1\cdots\delta_{i-1})(d_{j+1}\cdots d_M)
    }{
    \delta_1\cdots\delta_M
    }.
    \tag{*}
\end{equation*}

Let \(\eta\in(0,1)\) be the smaller root of
\(s^2-(2+h^2)s+1=0\), namely
\(\eta=\frac{2+h^2-h\sqrt{h^2+4}}{2}\), so that
\(\eta+\eta^{-1}=2+h^2\). Set
\(\Delta_k:=\delta_1\cdots\delta_k\), with \(\Delta_0:=1\). For
\(0\le k\le M-1\),
\[
    \Delta_k
    =
    \eta^{-k}\frac{1+\eta^{2k+1}}{1+\eta},
\]
and the final endpoint gives
\[
    \Delta_M
    =
    \eta^{-M}\frac{(1-\eta)(1-\eta^{2M})}{1+\eta}.
\]
Similarly, by reversing the path,
\[
    d_{j+1}\cdots d_M
    =
    \eta^{-(M-j)}
    \frac{1+\eta^{2(M-j)+1}}{1+\eta}.
\]
Substituting these formulas into \((*)\), and then using symmetry, yields
\[
    (A^{-1})_{ij}
    =
    \frac{
    \eta^{|i-j|+1}
    (1+\eta^{2(i\wedge j)-1})
    (1+\eta^{2(M-i\vee j)+1})
    }{
    (1-\eta^2)(1-\eta^{2M})
    }.
\]
Put \(\theta:=-\log\eta=2\operatorname{arsinh}(h/2)\).
For \(0<h\le1/4\), elementary bounds give
\(h/2\le\theta\le h\),
\(h/2\le1-\eta^2\le2h\), and
\(1-\eta^{2M}\ge1-e^{-1}>1/2\).
The two parenthesized factors in the numerator lie in \([1,2]\).
Also \(e^{-h}/2\ge e^{-1/4}/2>1/3\). The exact formula therefore gives
\[
    \frac{e^{-h|i-j|}}{3h}
    \le
    (A^{-1})_{ij}
    \le
    \frac{16e^{-h|i-j|/2}}{h}.
\]
Finally, \(H^{-1}=L^{-1}A_h^{-1}\). Since
\(h(M-1)<1\) and \(e<3\), the uniform lower bound follows.
\end{proof}

We now prove Lemma~\ref{thm:calibrated-zero-respecting}.

\begin{proof}[Proof of Lemma~\ref{thm:calibrated-zero-respecting}]
Write \(s:=|S|\) and \(H:=H_{\gamma,L}^{(M)}\). Since
\(\Gamma_p\le L_1\), we have
\(h^2=\gamma/L\le8c_{\rm hp}<1/16\). Hence \(h<1/4\),
\(M\ge4\), and the floor and ceiling in the construction are valid.
Elementary rounding bounds give
\[
\frac1{2h}\le s\le\frac1h,\qquad
j_\star=\lfloor M/2\rfloor\ge\frac{M}{3}\ge\frac1{3h}.
\]
The Green-kernel lemma and \(H^{-1}\mathbf 1=\gamma^{-1}\mathbf 1\)
give
\[
\begin{aligned}
&\frac1{9Lh}\le\rho\le\frac{16}{Lh},
&&\frac{s}{9WLh}\le\alpha\le\frac1{W\gamma},\\
&\frac{s^2}{9W^2Lh}\le\beta
\le\frac{\|a\|^2}{\gamma},
&&\|a\|^2=\frac{s}{W^2}.
\end{aligned}
\]
Using \(Lh^2=\gamma\) and
\(W=C_W\epsilon/(\gamma\sqrt h)\) yields the four displayed bounds in
the lemma. In particular, \(\alpha>0\).

\paragraph{Tolerance and the integer phase count.}
The tolerance hypothesis, \(\ell_1=360\ge1\), and \(c_\epsilon<1\)
imply
\[
L_1,\ \frac{L_1}{\ell_1}
\ge\frac{\epsilon^2}{c_\epsilon\Delta},\qquad
\left(\frac{L_q\epsilon^{q-1}}{\ell_q}\right)^{1/q}
\ge\frac{\epsilon^2}{c_\epsilon\Delta}
\quad(2\le q\le p).
\]
Indeed, the \(q\)-th tolerance inequality gives the stronger factor
\(c_\epsilon^{-(1+1/q)}\) for \(q\ge2\). Consequently
\[
\Gamma_p\ge\Gamma_p(1)
\ge\frac{\epsilon^2}{c_\epsilon\Delta},\qquad
\frac{\Delta\gamma}{\epsilon^2}
\ge\frac{c_{\rm hp}}{c_\epsilon}=5184C_W^2.
\]
For \(x:=c_N\Delta\gamma/\epsilon^2\), this gives \(x\ge5184/720>7\).
Thus \(N=\lfloor x\rfloor\ge7\) and \(N\ge x/2\).

\paragraph{Smoothness and initial gap.}
By the explicit Green bounds,
\[
\frac{2\|a\|}{\alpha}\le\frac{36\gamma}{\sqrt h},\qquad
\frac{\rho\|a\|^2}{\alpha^2}\le5184\gamma,\qquad
\frac{\ell_1r^2\|a\|^2}{\beta}\le36\ell_1r^2\gamma.
\]
Because \(\gamma\le c_{\rm hp}L_1\),
\(\ell_1r^2\gamma\le c_{\rm hp}L_1\), and
\(\gamma/\sqrt h=\gamma^{3/4}L^{1/4}\),
Lemma~\ref{thm:boundedness-smoothness} yields
\[
\begin{aligned}
    \operatorname{Lip}\!\left(\nabla f_{\Delta,L_{1:p},\epsilon}\right)
    &\le \gamma+4L+36\gamma/\sqrt h
       +5184\gamma+36\ell_1r^2\gamma\\
    &\le\left(\frac12+5221c_{\rm hp}
       +\frac{36c_{\rm hp}^{3/4}}{8^{1/4}}\right)L_1
      <\frac23L_1<L_1.
\end{aligned}
\]
For \(q=2,\ldots,p\),
\[
    \operatorname{Lip}\!\left(D^q f_{\Delta,L_{1:p},\epsilon}\right)
    \le\frac{36\ell_qr^{3-q}\gamma^q}{C_W^{q-1}\epsilon^{q-1}}
    \le\frac{36c_{\rm hp}^q}{C_W^{q-1}}L_q<L_q.
\]
The initial gap bound in Lemma~\ref{thm:boundedness-smoothness}, the phase
count, and \(\epsilon^2/\gamma\le\Delta/(5184C_W^2)\) give
\[
\begin{aligned}
f_{\Delta,L_{1:p},\epsilon}(0)
-\inf f_{\Delta,L_{1:p},\epsilon}
&\le\frac{10N}{\beta}+\frac{\rho}{2\alpha^2}\\
&\le360C_W^2\frac{\epsilon^2N}{\gamma}
  +2592C_W^2\frac{\epsilon^2}{\gamma}\\
&\le360c_NC_W^2\Delta+\frac{2592}{5184}\Delta
=\Delta.
\end{aligned}
\]
Thus \(f_{\Delta,L_{1:p},\epsilon}\in
\mathcal F_{1:p}(\Delta,L_1,\ldots,L_p)\).

\paragraph{Zero-respecting revelation and complexity.}
The Green bounds also give
\[
\sqrt\beta\le\frac{\sqrt\gamma}{C_W\epsilon},\qquad
\frac{\beta\sqrt\rho}{\alpha}
\le\frac{72\sqrt\gamma}{C_W\epsilon}.
\]
Thus Lemma~\ref{thm:gradient-lower-bound} gives, whenever
\(a^\top z_N=0\),
\[
\|\nabla f_{\Delta,L_{1:p},\epsilon}(z)\|
\ge\frac{C_W}{292}\epsilon>\epsilon.
\]
Let \(\mathsf A\) be any
first-order zero-respecting method and let \(z^{(t)}\) be its iterates on
\(f_{\Delta,L_{1:p},\epsilon}\). Since
\(\operatorname{supp}(a)\subseteq\{j_\star,\ldots,M\}\),
Lemma~\ref{thm:first-order-zero-chain} gives the hiding condition
\(a^\top z_N^{(t)}=0\) for every \(t<Nj_\star\).
Consequently no such iterate is \(\epsilon\)-stationary. Since
\(N\ge c_N\Delta\gamma/(2\epsilon^2)\) and
\(j_\star\ge1/(3h)\),
\[
    T_\epsilon\bigl(\mathsf A,f_{\Delta,L_{1:p},\epsilon}\bigr)
    \ge
    Nj_\star
    \ge\frac{c_N}{6}\frac{\Delta\gamma}{h\epsilon^2}
    =c_T\Delta L_1^{1/2}\Gamma_p^{1/2}\epsilon^{-2}.
\]
\end{proof}

\bibliographystyle{ims}
\bibliography{reference}
\end{document}